\documentclass{article}

\usepackage{natbib}
\usepackage{times}
\usepackage{helvet}
\usepackage{courier}
\usepackage{balance}

\usepackage{ifthen}
\usepackage{latexsym}
\usepackage[centertags]{amsmath}
\usepackage{amssymb}
\usepackage{amsfonts}
\usepackage{amsthm}
\usepackage{array,xspace,multirow,hhline,tikz,tabularx,booktabs,fixltx2e}
\usepackage{optparams,xspace}
\usepackage{enumitem}
\usepackage{xcolor}
\usepackage{subcaption}
\usepackage{graphicx}
\usepackage{framed}
\usepackage{wasysym}

\makeatletter
\newtheorem*{rep@theorem}{\rep@title}
\newcommand{\newreptheorem}[2]{%
\newenvironment{rep#1}[1]{%
 \def\rep@title{#2 \ref{##1}}%
 \begin{rep@theorem}}%
 {\end{rep@theorem}}}
\makeatother
\newreptheorem{theorem}{Theorem}
\newreptheorem{lemma}{Lemma}
\newreptheorem{proposition}{Proposition}%

\newtheorem{theorem}{Theorem}%
\newtheorem{lemma}{Lemma}%
\newtheorem{corollary}{Corollary}%
\newtheorem{example}{Example}%
\newcommand{\qmin}{q^-}
\newcommand{\qmax}{q^+}
\newcommand{\qnew}{q'}
\newcommand{\qold}{q}

\RequirePackage{ifthen,calc}
\newsavebox{\ffbox}\newlength{\ffboxlen}
\newcommand{\todo}[1]{%
  {\sbox{\ffbox}{\textbf{TODO:}\ \textit{{#1}}\ \textbf{:ODOT}}
    \settowidth{\ffboxlen}{\usebox{\ffbox}}
		\addtolength{\ffboxlen}{-5mm}
    \ifthenelse{\ffboxlen>\linewidth}{%
      \noindent\marginpar{$>>>>$}\textbf{TODO:}\ \textit{{#1}}\ \textbf{:ODOT}\marginpar{$<<<<$}}{%
      \noindent\marginpar{$>><<$}\textbf{TODO:}\ \textit{{#1}}\ \textbf{:ODOT}}}}
	  
\title{Crowdsourced Outcome Determination in Prediction Markets}

\author{
Rupert Freeman\\
Duke University\\
{\small\tt rupert@cs.duke.edu}
\and
S\'{e}bastien Lahaie\\
Microsoft Research\\
{\small\tt slahaie@microsoft.com}
\and
David M. Pennock\\
Microsoft Research\\
{\small\tt dpennock@microsoft.com}
}

\date{}

\begin{document}

\maketitle

\begin{abstract}
  A prediction market is a useful means of aggregating information about a future event. To function, the market needs a trusted entity who will verify the true outcome in the end. Motivated by the recent introduction of decentralized prediction markets, we introduce a mechanism that allows for the outcome to be determined by the votes of a group of arbiters who may themselves hold stakes in the market. Despite the potential conflict of interest, we derive conditions under which we can incentivize arbiters to vote truthfully by using funds raised from market fees to implement a peer prediction mechanism. Finally, we investigate what parameter values could be used in a real-world implementation of our mechanism.
\end{abstract}

\section{Introduction}
\label{sec:introduction}

Prediction markets are commonly used to elicit information about some future event. The market operates by allowing participants to buy and sell securities which pay off according to the outcome of the event, and participants with an informational edge are able to place profitable trades when the market price disagrees with their own forecast. Through this trading process, the market price can be construed as a consensus forecast of the underlying event probability. Prediction markets have proven effective at forecasting events in a variety of domains, including business and politics~\citep{spann2003internet,berg2006iowa}.

A key challenge in implementing and scaling prediction markets is the question of outcome determination (i.e., closing markets for events). Traditional prediction markets are centralized, in the sense that there exists a trusted center that creates markets, oversees transactions, and closes the market appropriately. The trusted center is a bottleneck for defining and closing markets, limiting the scope of what can be predicted. However, there has recently been a rise of interest in \emph{decentralized} prediction markets, where any user may create a market for an event. The markets are closed by consensus among a group of arbiters rather than by a center.

A decentralized platform removes the requirement for a highly trusted center, but it also allows each arbiter to directly influence the outcome of the market, in much the same way that agents may deliberately manipulate an event due to their own stake in the market; this is known as \emph{outcome manipulation} \citep{shi2009prediction,berg2006iowa,chakraborty2016trading}. Additionally, by allowing anyone to create a market, there is no longer any guarantee that all questions will be sensible, or even have a well-defined outcome.
In this paper, we propose a specific prediction market mechanism with crowdsourced outcome determination that addresses several challenges faced by decentralized markets of this sort.

First is the issue of \emph{outcome ambiguity}. At the time the market closes, it might be unreasonable to assign a binary value to the event outcome due to lack of clarity in the outcome. In a centralized market, it may be possible to postpone the market closing date to allow for rare cases of ambiguity, but it is not clear who should make such decisions in a decentralized marketplace. Therefore, it may be more fitting to allow outcomes to be non-binary, reflecting some level of disagreement. Outcomes in our mechanism are determined by the fraction of arbiters that report an event to have occurred. This also guarantees that every market is well-defined, by having traders explicitly trade on their expectations of the arbiter reports, not the actual event.

Second is the problem of \emph{peer prediction}. For the credibility of the market, it is essential that arbiters are incentivized to truthfully report their opinion as to the realized outcome. If, for instance, we reward arbiters for agreeing with the majority opinion, then they are forced to anticipate the reports of other arbiters, not their independent opinion. We address this problem by making a technical change to an existing peer prediction mechanism, the 1/prior mechanism.

Third is the inherent \emph{conflict of interest} that arises by combining prediction markets and peer prediction mechanisms. Even if arbiters can be incentivized to report truthfully in isolation, there is no way to prevent them also having a stake in the market. An arbiter holding securities that pay off in a particular event will be incentivized to report that the event has occurred, even if they do not genuinely believe it to be the case, as long as they have a chance of affecting the market outcome. We tackle this issue by charging a trading fee that is later used to pay the arbiters. We show that, as long as each agent is responsible for a limited fraction of trading, and questions are clear enough, realistic trading fees can fully subsidize truthful reporting on the part of the arbiters.

\paragraph{Related Work.}
This work is inspired by decentralized prediction markets based on crypto-currencies, including Truthcoin, Gnosis, and especially Augur~\citep{peterson2015augur}. As in Augur, our mechanism consists of a prediction market stage and an arbitration stage, with trading fees from the market stage subsidizing the arbitration. The details of the mechanisms differ in both stages, however, and Augur includes additional layers of complexity such as a reputation system.
While this complexity does provide additional protection against attack,~\citeauthor{peterson2015augur} do not obtain any theoretical guarantees or justification for their chosen parameters.
\citeauthor{clark2014decentralizing} also discuss outcome determination in crypto-based prediction markets, among several other implementation aspects.

Our work is most closely related to that of~\citet{chakraborty2016trading}, who consider a model where two agents participate in a prediction market whose outcome is determined by a vote among the two agents. Our model extends theirs by allowing an arbitrary number of traders, and not requiring that all traders are arbiters. Further, we take a mechanism design approach, obtaining an incentive compatible mechanism, while \citeauthor{chakraborty2016trading} focus on analyzing the equilibrium of a simple (to play) trading-voting game, with no peer prediction mechanism in the voting phase to incentivize truthful voting. Recent work by~\citet{witkowski2017proper} also looks at a combination of forecasting and peer prediction, although the forecasts in their paper are elicited via proper scoring rules, rather than prediction markets. 

The work of~\citet{bacon2012predicting} is similar in spirit to ours, as is the literature on outcome manipulation mentioned previously, but in all cases the concrete setting is quite different. We also draw heavily on existing literature in prediction markets~\citep{hanson2003combinatorial,chen2007utility,chen2010new} and peer prediction~\citep{miller2005eliciting,prelec2004bayesian,witkowski2012robust}; \citet{chen2010designing} survey these topics.

\section{Preliminaries}
\label{sec:preliminaries}

Let $N$ be a set of agents and let $A \subset N$ be a small set of distinct and verifiable \emph{arbiters}. Let $m=|A|$ denote the number of arbiters. The agents are anonymous in the sense that one cannot verify whether a trade is placed by an arbiter or non-arbiter, and whether several trades are placed by the same agent. Let $X$ be a binary event with some realized outcome in $\{ 0,1 \}$. We are interested in setting up a prediction market for the outcome of $X$, with the resolution of the market decided upon by the arbiters.

\paragraph{Prediction markets.}
\color{black}
We consider prediction markets implemented via a \emph{Market Scoring Rule (MSR)}, where the underlying scoring rule is strictly proper~\citep{hanson2003combinatorial,chen2007utility}. Every strictly proper MSR can be implemented as a market maker based on a \emph{convex cost function}. Under this market structure, agents trade shares of a security with a centralized market maker, who commits to quoting a buy and sell price for the security at any time. The security payout is contingent on the outcome of $X$. In the usual implementation, one share of the security pays out \$1 in the event that $X=1$, and \$0 otherwise.

Let $q$ denote the total number of outstanding shares owned by the agents (note that $q$ can be negative, in the case that more shares have been sold than bought). The market maker charges trades according to a convex, differentiable, and monotonically increasing function $C$.
An agent wishing to buy $q' - q$ securities pays $C(q')-C(q)$. Negative payments encode a transaction where securities are sold back to the market maker.
The instantaneous price of the security is given by $p = \frac{dC}{dq}$.
Because the market maker always commits to trading, it may run a loss once the outcome is realized and the securities pay out, but the loss is bounded.

In practice, the cost function is also tuned using a \emph{liquidity parameter} $b$, via the transformation $C_b(q) \equiv b\, C(q/b)$. A higher setting of $b$ results in lower price responsiveness, in the sense that the price will change less for a fixed dollar trading amount. It also results in a higher worst-case loss bound for the market maker.
Unless otherwise stated, our results assume that each agent participates in the market only once. The mechanism and results extend to situations in which agents can participate more than once, and we highlight these extensions where relevant throughout the paper.

\paragraph{Peer prediction.} Peer prediction mechanisms are designed to truthfully elicit private information from a pool of agents via a reward structure that takes advantage of information correlation across agents. After the realization of $X$, each arbiter $i$ receives either a positive or negative signal $x_i$, which we denote by $x_i=1$ and $x_i=0$ respectively. 
Let $\mu$ be the prior probability that an agent receives a positive signal. Let $\mu_1$ (resp. $\mu_0$) be the probability that, given agent $i$ receives a positive (resp. negative) signal, another randomly chosen agent also receives a positive signal.\footnote{Our analysis will assume that $\mu_1$ and $\mu_0$ are common across agents, but this is not a strict requirement. If we allow each agent to have distinct updates $\mu_1^i, \mu_0^i$, then we can let $\mu_1 = \min_i \mu_1^i$, corresponding to the minimum update given $\hat{x}_i=1$, and similarly $\mu_0= \max_i \mu_0^i$.}
In a standard peer prediction belief model, the quantities $\mu_1$ and $\mu_0$ can be calculated given $\mu$ and the signal beliefs $P(x_i=1 | X=1)$ and $P(x_i=1 | X=0)$; \cite{witkowski2014robust} provides an overview.
Assuming common information is not always reasonable, but it is natural in our setting to assume that agents take the closing price of the prediction market as their prior (if not, then they can profit in expectation by participating in the market). The probabilities $\mu_1$ and $\mu_0$ are specific to the nature of the event $X$.

 The peer prediction mechanism of interest in this work is the 1/prior (``one over prior") mechanism \citep{witkowski2014robust,jurca2008incentives,jurca2011incentives}. The 1/prior mechanism first asks each agent for their signal report $\hat{x}_i$. Then, every agent $i$ is randomly paired with a peer agent $j \neq i$, and paid
\[ u(\hat{x}_i, \hat{x}_j) = \begin{cases}
      k \mu & \text{if }\hat{x}_i = \hat{x}_j=0 \\
      k (1-\mu) & \text{if }\hat{x}_i=\hat{x}_j=1 \\
      0 & \text{if }\hat{x}_i \not= \hat{x}_j,
   \end{cases}
\]
where $k$ is some positive constant that can be freely adjusted to scale the payments received by the arbiters. Truthfully reporting $\hat{x}_i=x_i$ is an equilibrium provided that $\mu_1 \ge \mu \ge \mu_0$ \citep{frongillo2016geometric}. This is a natural condition that we will assume throughout the paper---receiving signal $x_i=1$ should not \emph{decrease} $i$'s estimate that another agent $j$ also receives signal $\hat{x}_j=1$. We also assume that at least one of the inequalities is strict, so that $\mu_1 > \mu_0$; this condition is known as \emph{stochastic relevance}. Via a simple adaptation of the corresponding proof for the 1/prior mechanism, it can be shown that truthful reporting remains an equilibrium if $\mu$ is replaced by some other constant $c$ with $\mu_0<c<\mu_1$ in the payments; we will exploit this fact to adapt the 1/prior mechanism for our purpose.

We call the quantity $\delta = \mu_1-\mu_0$ the \emph{update strength}. This quantity is specific to the instance and, roughly speaking, measures how strongly positively correlated the signals are across arbiters. The update strength is high if, after receiving a positive (negative) signal, an arbiter believes that another given arbiter receives a positive (negative) signal with high probability. For instance, if event $X$ is ``Will the Cleveland Cavaliers win the 2016 NBA playoffs?'' then we would expect $\delta \approx 1$, since any arbiter reaching a conclusion about the outcome of the series (by watching it live, reading in the news, etc.) would strongly expect any other arbiter to reach the same conclusion. On the other hand, a question like ``Will a Presidential candidate tell a lie in the televised debate?'' is considerably more open to interpretation, and we would expect it to have a smaller value of $\delta$. If an arbiter believes a candidate to have lied, it is not necessarily the case that another arbiter believes the same.


\section{Mechanism}

\begin{figure}[!t]
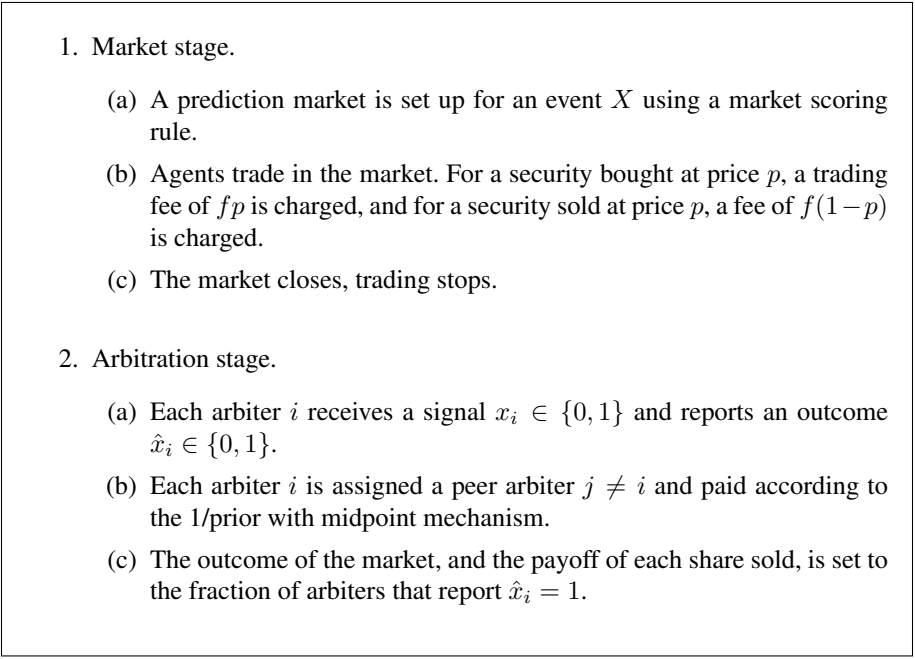

\begin{framed}
\begin{enumerate}
	\item Market stage.
	\begin{enumerate}
		\item A prediction market is set up for an event $X$ using a market scoring rule.
		\item Agents trade in the market. For a security
                  bought at price $p$, a trading fee of $fp$ is
                  charged, and for a security sold at price $p$, a
                  fee of $f(1-p)$ is charged.
		\item The market closes, trading stops.
	\end{enumerate}
	\medskip
	\item Arbitration stage.
	\begin{enumerate}
		\item Each arbiter $i$ receives a signal $x_i \in \{ 0,1 \}$ and reports an outcome $\hat{x}_i \in \{ 0,1 \}$.
		\item Each arbiter $i$ is assigned a peer arbiter $j \neq i$ and paid according to the 1/prior with midpoint mechanism.
		\item The outcome of the market, and the payoff of each share sold, is set to the fraction of arbiters that report $\hat{x}_i=1$.
	\end{enumerate}
\end{enumerate}
\end{framed}
\caption{Prediction market with outcome determined using peer prediction.}
\label{fig:pmpp}
\end{figure}

A step by step description of our mechanism is given in Figure~\ref{fig:pmpp}.
The mechanism runs a prediction market where the outcome is determined by a vote among arbiters. The arbiters' signals should be interpreted as the information they receive regarding the outcome of $X$: checking news sources, observing events, their own opinions, etc. To ensure that arbiters truthfully report their information, we incentivize them via a peer prediction mechanism.\footnote{Each arbiter makes his report without knowledge of the report of any other arbiter; for instance, the reports could be made simultaneously.} In both stages we implement non-standard versions of existing mechanisms, which we detail in the following.

\subsubsection{Market stage}

We make use of an MSR with non-binary outcome. The outcome takes a value $\hat{X} \in [0,1]$ corresponding to the fraction of arbiters that report $\hat{x}_i=1$. Each share sold pays off $\hat{X}$. Observe that this fundamentally changes the value of a security to a market participant: in a standard prediction market, an agent's value for a security is his subjective probability that event $X$ occurs, while in our market his value is the fraction of arbiters that he expects to report $\hat{x}_i=1$. However, given the agent's valuation for a security, his incentives in both markets are similar. A risk-neutral, non-arbiter agent will trade shares until the market price matches the security's expected payoff, or the agent's budget is exhausted.

This change to the payoff structure has two advantages. First, it ensures that any question has a well-defined and unambiguous outcome, avoiding problems with badly worded questions. This is important in any situation where users are allowed to generate markets. Second, any market with a binary outcome that relies on arbitration must have a point of `discontinuity', where a change in report from a single arbiter results in the value of a security changing by \$1.\footnote{To see this, consider the case where all arbiters report $\hat{x}_i=1$, and flip one report at a time to $\hat{x}_i=0$. One of these flips must change the outcome from $\hat{X}=1$ to $\hat{X}=0$.} There will therefore always be situations where, given the reports of the other arbiters, a single arbiter completely controls the market outcome. If this arbiter also has a significant stake in the market, this creates a very large incentive problem. By utilizing non-binary outcomes, a single arbiter can only change the value of each security by at most \$$1/m$.

Our mechanism imposes trading fees. Theoretical models of prediction markets do not typically incorporate trading fees (an exception is the work of \citeauthor{othman2013practical}, where a fee in the form of a bid-ask spread is used to allow liquidity to increase over time), but they are standard in real-world implementations.
%
To understand how the fee is implemented, it is important to distinguish between transactions (buy or sell) where an agent increases its position (in terms of risk), and transactions where it liquidates its position. The trading fee that we implement can be seen as a fee on the worst-case loss incurred by an agent: the fee is on $p$ when a new security is bought, and $1-p$ when a security is sold short (because it may pay out \$1). However, no fee is levied when an agent sells back a share that it holds, or buys back a share that was previously sold short---these are liquidation transactions.

%
%
The trading fee serves two distinct purposes in our mechanism. First, it allows us to raise funds which can then be used to pay arbiters. Even assuming that arbiters behave honestly (in the absence of a sophisticated peer prediction mechanism), they still need to be compensated for the time spent looking up the outcome of $X$ and reporting it to the mechanism. This can, in principle, be achieved by any of a number of fee structures.

Second, the fee provides natural bounds on the value of any given security. Even if an event is certain to occur, with a fee of $f = 2\%$ an agent who moves the market price to (say) $99\cent$ actually pays a marginal cost of $\$0.99 \cdot 1.02 >\$1$ (see Lemma~\ref{lem:min-price} for an exact bound).
The multiplicative fee effectively bounds the price of the security away from 0 and 1.
Thus, it is impossible for an agent to buy securities at an arbitrarily cheap price, which allows us to bound the number of securities, and therefore maximum payout, of any agent with a fixed budget $B$. We note that there are other reasonable fee structures which do not provide such a lower bound on the price. For example, if the agents only pay a fee on any profit they gain from the market, then the price of an event that is certain to happen will still converge to 1.

\subsubsection{Arbitration stage}

The main challenge in our setting is to incentivize arbiters to truthfully report their signal regarding the realized value of $X$. In the absence of any conflict of interest, this is a simple peer prediction problem. Since the closing price of the market gives us a natural common prior on the probability that a given arbiter receives signal $x_i=1$, it is natural to use the 1/prior mechanism. For prior signal probability $\mu$, the 1/prior mechanism uses the fact that $\mu_1 \ge \mu \ge \mu_0$ to guarantee that truthful reporting achieves higher payoff than misreporting.
However, as $\mu_1$ approaches $\mu$, the payoff for truthfully reporting signal $\hat{x}_i=1$ approaches the payoff for misreporting $\hat{x}_i=0$. In isolation, there is still no reason to misreport, but if the arbiter has some stake in the market then it may be worthwhile to incur a small misreporting loss to achieve other gains. The following example illustrates this issue.

\begin{example} \label{ex:1/p-fails}
	Consider a prediction market for the event ``Will the Democratic presidential candidate be leading the Republican presidential candidate in the polls at the end of the month?'' Suppose it is known that 10\% of arbiters only check conservative news sources, which always report that the Republican candidate is ahead, and another 10\% only check liberal news sources, which always report the opposite. Suppose the market closes at $\mu = 0.89$. Consider an arbiter $j$ who checks a (moderate) news source and finds that the Democratic candidate is ahead (i.e., $x_j=1$). Since it is still the case that 10\% of the arbiters will certainly receive signal $x_i=0$, the updated belief $\mu_1$ remains no higher than 0.9. That is, the update is very small, and the expected profit from reporting $\hat{x}_j=1$ is also small. If $j$ has bet against the outcome (i.e., sold some securities to the market maker), it could be in his interest to lie and report $\hat{x}_j=0$.

	 However, if the moderate news site had reported that the Republican candidate was leading (i.e., $\hat{x}_j=0$), the updated belief $\mu_0$ could be quite small, even in the range of 0.1 (since most arbiters check moderate sources). Now $j$ has a lot to gain from reporting $\hat{x}_j=0$. Therefore, $j$ would have to hold a relatively large number of shares for misreporting to outweigh the expected profit from the 1/prior mechanism.
	\end{example}
\noindent
	Example~\ref{ex:1/p-fails} stems from an asymmetry in update strength,  leading to potentially different incentives for arbiters depending on which signal they receive. We modify the mechanism, making the update strength symmetric. Given that we know the updated beliefs $\mu_1$ and $\mu_0$, we can pay arbiters according to the 1/prior mechanism but use the value $(\mu_1+\mu_0)/2$ instead of the prior, $\mu$. We call this the \emph{1/prior with midpoint} mechanism. Using the midpoint guarantees that the incentives for arbiters are the same regardless of the signal they receive. For the arbiter with the greatest incentive to misreport, using the 1/prior with midpoint mechanism (weakly) decreases his incentive to misreport over the standard 1/prior mechanism, allowing us to achieve better bounds in our worst-case analysis.

\subsection*{Analysis}

In this section, we derive conditions for truthful reporting ($\hat{x}_i = x_i$) to be a best response, given that all other arbiters report truthfully.
The main restriction we require is an upper bound $B$ on the total budget any given arbiter spends in the market---without such a bound, an arbiter could have an arbitrarily large incentive to manipulate the market's outcome. Thus, $B$ appears as a parameter in our analysis.

Arguably, an arbiter confident in their ability to manipulate a market outcome could procure enough funds as to have a very large budget, especially relative to a small market. However, in current decentralized prediction markets, each arbiter arbitrates only a small fraction of markets. As long as the assignment of arbiters to markets is done \emph{after} the market closes, there is no way for manipulators to target a specific market. For this reason, we believe that manipulations are most likely to be of a form where arbiters participate honestly in the first stage, but, if they happen to be assigned to arbitrate a market that they also participated in, may be able to gain by not reporting truthfully, rather than arbiters mounting deliberate high-budget attacks in the market stage. Of course, our analysis is not specific to that particular interpretation, but we do consider it a compelling argument in favor of using a budget bound in our analysis.

Intuitively, we need to scale the payments made to arbiters in the arbitration stage by a sufficiently large $k$ so that the increased payoff for truthful reporting in this stage overwhelms the gains from manipulating the outcome.

\begin{lemma} \label{lem:peer-prediction-multiplier}
	Let $n_i$ be the number of securities held by arbiter $i$. Then truthfully reporting $\hat{x}_i =x_i$ is a best response for arbiter $i$, given that all other arbiters report truthfully, if and only if
	\[ k \ge \frac{2|n_i|}{m\delta}. \]
\end{lemma}

\begin{proof}
	We prove the case where $n_i>0$; the case for $n_i<0$ is symmetric.
	The total payoff for arbiter $i$ is the sum of the payoffs from the market phase and the arbitration phase. Fixing the reports of the other arbiters, the market payout for $i$ is higher when $i$ reports $\hat{x}_i=1$. And, in expectation, the payoff for $i$ in the arbitration phase is higher for truthful reporting than for lying. Thus, the only problematic case is when $x_i=0$, but $i$ may wish to report $\hat{x}_i=1$.

	So suppose that $x_i=0$. The expected payoff for truthfully reporting $\hat{x}_i=0$, assuming all other arbiters truthfully report their signal, is
	\begin{equation} \label{eqn:truthful-payoff}
	n_i\mu_0\frac{m-1}{m} + (1-\mu_0) k \, \frac{\mu_0+\mu_1}{2}.
	\end{equation}
	Here $\mu_0(m-1)$ is the expected number of arbiters that report signal $1$, and therefore $n_i\mu_0 (m-1)/m$ is $i$'s expected payoff from the market, while the remaining term is $1-\mu_0$, the probability of $i$'s peer agent also reporting 0, multiplied by the payment $i$ receives in this case.
	On the other hand, the expected payoff for misreporting $\hat{x}_i=1$ is
	\begin{equation} \label{eqn:misreport-payoff}
		n_i\left( \mu_0\frac{m-1}{m}+\frac{1}{m} \right) + \mu_0 \, k \left( 1-\frac{\mu_0+\mu_1}{2} \right),
		\end{equation}
	where the extra $1/m$ in the first term is due to the additional market payoff from $i$ reporting $\hat{x}_i=1$, and the latter term is now the probability of $i$'s peer agent reporting 1, multiplied by the payoff $i$ receives when this happens.

	We require that the expected payoff for reporting $\hat{x}_i=1$ be at most the expected payoff for truthfully reporting $\hat{x}_i=0$. Setting term~(\ref{eqn:misreport-payoff}) to be at most term~(\ref{eqn:truthful-payoff}) and simplifying yields the result.
\end{proof}
\noindent
This characterization requires an upper bound on the number of securities that any single agent owns. In itself this is an unsatisfying restriction; however, we can think about it in terms of the size of the fee, $f$, and the amount of money that any single arbiter spends in the market, $B$.
For fixed fee $f$, let $\qmin$ and $\qmax$ be the number of outstanding securities such that the market prices are $p(\qmin) = f / (1+f)$ and $p(\qmax) = 1/(1+f)$ respectively. Note that these quantities depend on the liquidity parameter $b$ used in the cost function.

\begin{lemma} \label{lem:min-price}
	For fixed percentage fee $f$, the number of outstanding securities lies in the interval $[\qmin,\qmax]$.
\end{lemma}

\begin{proof}
	Suppose that some agent sells a security when there are already $\qmin$ outstanding. Then the marginal price is exactly $f/(1+f)$. When selling a security at this price, the agent receives $f/(1+f)$ from the mechanism but must pay a trading fee of $$f\left( 1-\frac{f}{1+f} \right) = \frac{f}{1+f}.$$ Thus the agent's net revenue from the sale is 0 (and the possibility remains that he must pay the mechanism in the event that $X$ occurs). Therefore no agent makes such a sale, and the number of outstanding securities never drops below $\qmin$.

	A similar argument shows that $q$ never exceeds $\qmax$. To buy a security when there are already $\qmax$ outstanding, an agent must pay a price of at least \$1, when the fee is included. 
\end{proof}

\noindent
Lemma~\ref{lem:min-price} provides us with the minimum and maximum number of outstanding securities at any time. As a corollary, we can derive the maximum number of securities that a single agent with budget $B$ is able to purchase or short sell. We interpret the budget as an upper bound on the worst-case loss that the agent is able to incur. When buying a security for price $p$, the worst-case loss is $p$, under outcome $X=0$. When selling a security for price $p$, the worst-case loss is $1-p$, under outcome $X=1$. Let $\phi_b^+(B)=C^{-1}_b(B + C_b(\qmin))-\qmin$. Define $q'$ implicitly by $B + C_b(q^+) - C_b(q') = q^+ - q'$, and let $\phi_b^-(B) = q'-\qmax$.

\begin{corollary} \label{cor:max-number-securities}
	At the end of the market stage, an agent $i$ with budget $B$ holds $n_i \in [\phi_b^-(B), \phi_b^+(B)]$ securities.
\end{corollary}

\begin{proof}
	We first show the upper bound. Given an existing number of outstanding securities, $\qold$, an agent is able to increase the number of outstanding securities to $\qnew$, where 
	\begin{equation} \label{eq:1}
	C_b(\qnew)-C_b(\qold)=B.
\end{equation}
For a fixed budget, the maximum number of securities that can be bought in a single transaction is in the case that $\qold$ is as small as possible; in our case, $\qold=\qmin$. Substituting into~\eqref{eq:1} gives
\[ \qnew = C^{-1}_b(B + C_b(\qmin)). \]
The number of securities held by $i$ is $\qnew-\qold=\qnew-\qmin$, which gives the upper bound in the statement of the corollary.

We now show the lower bound. Given an existing number of securities $\qold$, an agent is able to decrease the number of outstanding securities to $\qnew$, where 
	\begin{equation} \label{eq:2}
	B + C_b(\qold) - C_b(\qnew) = \qold - \qnew.
\end{equation}
The right hand side of~\eqref{eq:2} is the number of securities sold by the agent to the mechanism, and therefore the amount that he may be required to pay the mechanism in the case that $X=1$. The left hand side is exactly the funds that the agent is able to use to reimburse the mechanism: his budget, $B$, plus the amount paid to the agent by the mechanism for the securities, $C(\qold) - C(\qnew)$. The maximum number of securities that can be sold in a single transaction is in the case that $\qold=\qmax$; making this substitution in~\eqref{eq:2} yields the implicit formula for $q'$ in the definition of $\phi_b^-(B)$. The number of securities sold by $i$ is $\qnew-\qold=\qnew-\qmax$, which gives the lower bound in the statement of the corollary.
\end{proof}
\noindent
An interesting special case is the limit as $b \to \infty$. This corresponds to the market having zero price responsiveness, meaning that all securities are purchased at a fixed price. Conceptually, it is equivalent to the situation where agents participate in the market more than once. In that setting, an agent could wait until the market price reaches $\frac{f}{1+f}$, buy a small number of securities, then wait again until the price drops. An agent spending all their budget in this way can, in the extreme case, buy as if the market has infinite liquidity.

\begin{corollary} \label{cor:infinite-liquidity-max-number-securities}
	For an agent that spends at most $B$ dollars in a market with trading fee $f$ and infinite liquidity, $n_i$ lies in the range $\left[ -\frac{B(1+f)}{f}, \frac{B(1+f)}{f} \right]$.
\end{corollary}

\begin{proof}
	The minimum price for a single security is $\frac{f}{1+f}$, by Lemma~\ref{lem:min-price} and the definition of $\qmin$. Therefore, an agent with budget $B$ can buy at most $\frac{B(1+f)}{f}$, the upper bound in the corollary statement.

	The maximum price for a single security is $\frac{1}{1+f}$, by Lemma~\ref{lem:min-price} and the definition of $\qmax$. Thus, an agent selling a security has worst case loss at least $1-\frac{1}{1+f}=\frac{f}{1+f}$. So, an agent with budget $B$ can sell at most $\frac{B(1+f)}{f}$ securities, which yields the lower bound.
\end{proof}

\noindent
%
If every agent has budget at most $B$ in the market stage, we can combine the bounds from Corollaries~\ref{cor:max-number-securities} and~\ref{cor:infinite-liquidity-max-number-securities} and Lemma~\ref{lem:peer-prediction-multiplier} to determine the minimum payment that guarantees truthful reporting in the arbitration phase.

\begin{theorem} \label{thm:peer-prediction-multiplier}
	Given that all other arbiters report truthfully, truthful reporting is a best response for arbiter $i$ if
	\[ k \ge \frac{2 \max \{|\phi_b^-(B)|,|\phi_b^+(B)|\}}{m\delta}. \]
	In the case that agents may participate in the market many times, truthful reporting requires that
	\[ k \ge \frac{2B(1+f)}{fm\delta}. \]
\end{theorem}

\begin{proof}
	The theorem follows directly from substituting the lower bound on $n_i$ from Corollary~\ref{cor:max-number-securities} and Corollary~\ref{cor:infinite-liquidity-max-number-securities} into the inequality from Lemma~\ref{lem:peer-prediction-multiplier}.
\end{proof}

\noindent
Therefore, fixing an agent budget $B$ and a trading fee $f$, we know how large one needs to make the payments in the arbitration phase in order to incentivize truthful reporting. We now take a global view, and examine the total funds required to incentivize all arbiters to report truthfully.

\begin{lemma} \label{lem:total-payment-bound}
	The total payment made to the arbiters is at most $m k$. We can implement a truthful equilibrium with total payment at most
	\[ \frac{2 \max \{|\phi_b^-(B)|,|\phi_b^+(B)|\}}{\delta}. \]
	In the case that agents may participate in the market many times, we require total payment at most
	\[ \frac{2B(1+f)}{f\delta}. \]
\end{lemma}

\begin{proof}
As $0 \leq \mu_0, \mu_1 \leq 1$, their mean also lies between 0 and 1, and therefore each arbiter's payment in the 1/prior with midpoint mechanism is at most $k$.  Thus the total payment to the arbiters is at most $m k$, which proves the first part.
Combining this with the bounds on $k$ from Theorem~\ref{thm:peer-prediction-multiplier} yields the second part.
	\end{proof}

\noindent
Now that we have an expression for the total amount needed to pay the arbiters, we can determine a suitable value for the fee $f$ so that the mechanism does not need any outside subsidy to finance these payments. Let $c_i$ denote the total cost paid by agent $i$ to the mechanism (so $c_i$ is negative if agent $i$ sells securities). Define $M$ by
\[ M = \sum_{i: n_i >0} c_i + \sum_{i: n_i < 0} (n_i+c_i). \]
$M$ can be interpreted as the sum of the worst-case losses of the agents.
By definition, the total fee revenue collected by the mechanism is $fM$. The mechanism is guaranteed to generate enough fees to incentivize truthful reporting if the revenue is at least as large as the total payment required for the arbiters. We state this result as a theorem.

\begin{theorem} \label{thm:main}
The mechanism generates enough fee revenue to pay the arbiters without requiring any outside subsidy if
\begin{equation} \label{eq:main}
fM \ge \frac{2 \max \{|\phi_b^-(B)|,|\phi_b^+(B)|\}}{\delta}.
\end{equation}
If agents may participate in the market many times, then we require that
\begin{equation} \label{eq:main2}
fM \ge \frac{2B(1+f)}{f\delta}.
\end{equation}
\end{theorem}
\noindent
Observe that inequality~\eqref{eq:main2} aligns with intuition. An increase in total trader spend $M$, or the trading fee $f$, makes it easier to incentivize the arbiters to report truthfully since the market collects more revenue. Likewise, an increase in $\delta$ helps us satisfy the inequality, since a large update strength increases the incentive for arbiters to report truthfully to the peer prediction mechanism. However, a large value of $B$ increases the stake that any single arbiter can have in the market, which in turn increases their payoff for misreporting.

An interesting feature of inequalities~\eqref{eq:main} and~\eqref{eq:main2} is the lack of any dependence on the number of arbiters $m$. One might expect that increasing the number of arbiters would be beneficial, since this reduces the influence that any one of them has on the market outcome. However, this is canceled out by the fact that as we add arbiters, the payment per arbiter decreases, so that we cannot incentivize them as strongly.

\section{Parameter Calibration}

\begin{figure*}[!t]
	\centering
	\begin{subfigure}{0.49\textwidth}
		\centering
		{\includegraphics[width=\columnwidth]{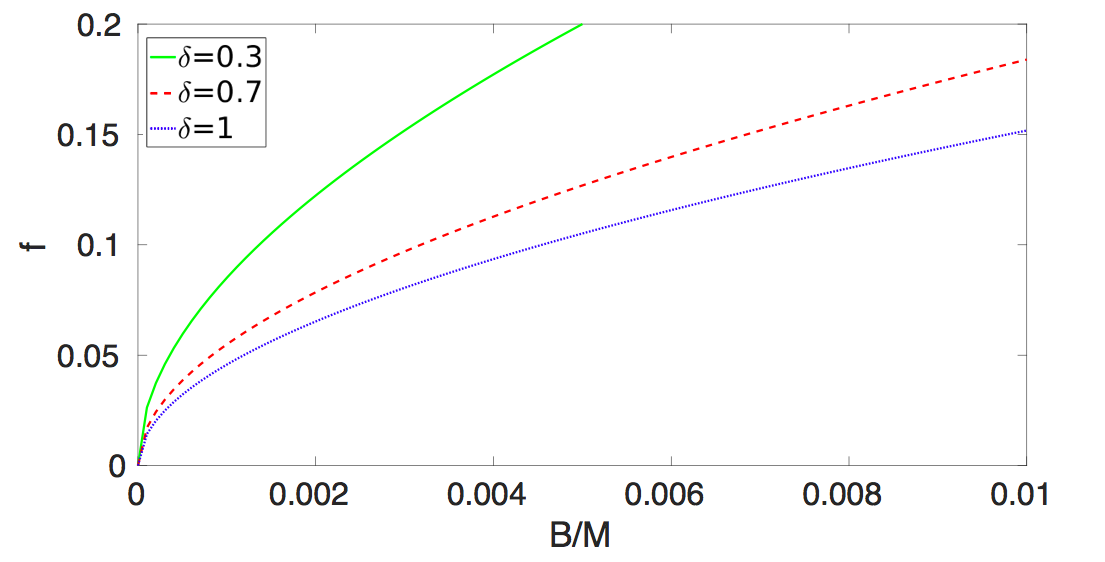}}
		\caption{Multiple entry case. \label{fig:fvsB}}
	\end{subfigure}
	\begin{subfigure}{0.49\textwidth}
		\centering
		\includegraphics[width=\columnwidth]{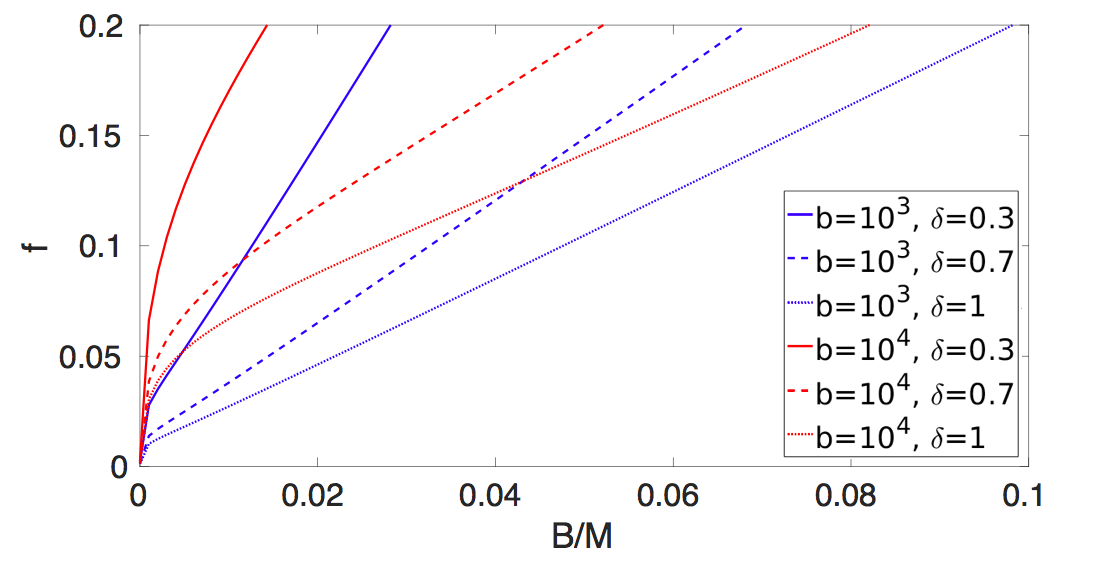}
		\caption{Single entry case. \label{fig:main}}
	\end{subfigure}
\caption{Minimum fee $f$ required to adequately incentivize arbiters, plotted as a function of $\frac{B}{M}$. In both cases, $M = 10^6$ is fixed. Relationships are shown for selected values of update strength $\delta$ and, in the right-hand plot, liquidity $b$.}
\label{fig:simulations}
\end{figure*}

In this section we investigate the constraints imposed by inequalities~\eqref{eq:main} and~\eqref{eq:main2}. The purpose of the exercise is to illustrate how Theorem~\ref{thm:main} can be used to inform the choice of fee $f$, and to confirm that realistic fees could be charged in practice to subsidize truthful arbitration. We consider the logarithmic market scoring rule (LMSR), which is the most common MSR used in practice. For the LMSR, the cost and price functions are
\begin{align*}
	C_b(q)=b \log(1+e^{q/b}), \qquad p(q)=\frac{e^{q/b}}{1+e^{q/b}}.
\end{align*}
By the symmetry of LMSR, $q^- = -q^+$ and $\phi_b^-(B) = -\phi_b^+(B)$. We will therefore solve for $\phi_b^+(B)$.
To find $q^-$, we set $p(q)=f/(1+f)$ and solve for $q$, which gives $q^-=b \log f$. Now, substituting the relevant components into the expression $\phi_b^+(B) = C_b^{-1}(B+C_b(q^-))-q^-$
leads to the following expression for inequality~\eqref{eq:main}:
\begin{equation} \label{eq:main-lmsr}
	fM \ge \frac{2b(\log ((1+f)e^{B/b}-1)- \log f)}{\delta}.
\end{equation}
In the case where we allow agents to participate multiple times, inequality~\eqref{eq:main2} remains unchanged.

We plot~\eqref{eq:main} and~\eqref{eq:main2} in Figure~\ref{fig:simulations}, considering their tight versions as equalities. First consider Figure~\ref{fig:fvsB}, which represents the worst-case scenario in which agents can enter multiple times and potentially spend their entire budget buying securities at minimum price $p^-$. Suppose that some entity is creating a prediction market for event $X$. Having decided on a question, the main decision is what value to set for $f$, typically in the 2-5\% range.
To do so, the market creator needs to first estimate a value for $\delta$, which will be determined by question clarity, whether the arbiters have reliable sources to check the outcome, and other such factors. Each line in the graph represents a specific value of $\delta$. With $\delta$ fixed, the market creator can estimate a value for $\frac{B}{M}$. This is the maximum proportion of money that any single arbiter will contribute to the market. We would expect $\frac{B}{M}$ to be small for markets that generate a lot of interest, while niche markets would be vulnerable to having a single agent contribute a large percentage of the total trade. Given these values, the creator can arrive at the smallest $f$ that is guaranteed to subsidize truthful reporting. From the graph, we see that in the case of a question where $\delta=1$ and $\frac{B}{M}=0.001$, we can subsidize the arbiter payment with a fee of approximately 4\%. This may seem large for a clear question with high participation, but we stress that this fee is based on a severe worst case where an agent is able to spend its entire budget purchasing securities at the minimum price.

Now consider Figure~\ref{fig:main}, which returns to the case where an agent only enters once, where liquidity now plays a role and we have to consider different values for parameter $b$. Figure~\ref{fig:main} includes two reasonable values for $b$, as well as three different values for $\delta$. We note that the situation looks considerably better for the market creator; indeed, the horizontal axis is now ten times larger indicating that we can now handle much smaller markets. When $\delta=1$, we can handle situations where a single agent can contribute as much as 2\% of the total trade with a fee of less than 5\%. Even for questions with $\delta$ as low as $0.3$, in a market with $b=1000$ and $\frac{B}{M}=0.005$ the fee can be set to approximately $5\%$.

\section{Conclusion}
\label{sec:conclusion}

This paper proposed and analyzed a mechanism where the outcome of an MSR prediction market is determined via a peer prediction mechanism among a set of arbiters. The mechanism relies on two key adaptations to incentivize truthful arbitration: market shares pay out according to the proportion of arbiters who vote affirmatively, instead of a binary payout, and peer prediction payments are based on the midpoint of the two possible posteriors, rather than the prior. We showed that, with this combination of adaptations, it is possible to charge a trading fee that fully subsidizes truthful arbitration. Calibration based on plausible values of question clarity and trading volume suggests that realistic fees of 5\% should be sufficient in practice.

While we have derived conditions under which truthful reporting is an equilibrium, there remains the possibility of the arbiters reporting according to uninformative equilibria that achieve higher payoff. This problem has recently been addressed in the peer prediction literature in situations where reporters complete several tasks instead of just one \citep{dasgupta2013crowdsourced,shnayder2016informed}; it may be worthwhile to apply these techniques to our setting. In practice, arbiters vote on many questions across time, which opens the possibility of using a reputation system to incentivize them to vote truthfully and accurately~\citep{peterson2015augur}. The interplay of the incentives from all these mechanisms is fertile ground for future research.

\section*{Acknowledgments} 

\begin{small}

This work was done in part while Freeman was at Microsoft Research, New York City. He also thanks NSF IIS-1527434 and ARO W911NF-12-1-0550 for support.

\bibliographystyle{natbib}

\end{small}

\clearpage

\color{black}
%
%
%
%
%
%

\end{document}